\documentclass[letterpaper,journal]{IEEEtran}
\usepackage{amsmath,amsfonts}
\usepackage{algorithm}
\usepackage{algpseudocode}
\usepackage{array}
\usepackage[caption=false,font=normalsize,labelfont=sf,textfont=sf]{subfig}
\usepackage{textcomp}
\usepackage{stfloats}
\usepackage{url}
\usepackage{verbatim}
\usepackage{graphicx}
\usepackage{cite}
\usepackage{hyperref}
\usepackage{bm}
\usepackage{booktabs}
\usepackage{tabularx}
\usepackage{makecell}
\usepackage{threeparttable}
\newcolumntype{Y}{>{\centering\arraybackslash}X}

\usepackage{amsthm}

\newtheorem{lemma}{Lemma}

\theoremstyle{remark}
\newtheorem{remark}{Remark}
\usepackage{xcolor}
\usepackage{orcidlink}
\usepackage{eso-pic}
\AddToShipoutPictureBG*{\AtPageUpperLeft{%
  \put(\LenToUnit{0.5\paperwidth},\LenToUnit{-0.07in}){%
    \makebox[0pt][c]{\raisebox{-\height}{\parbox{0.9\paperwidth}{\centering\tiny
    \copyright~2026 IEEE. Personal use of this material is permitted. Permission from IEEE must be obtained for all other uses, in any current or future media, including reprinting/republishing this material for advertising or promotional purposes, creating new collective works, for resale or redistribution to servers or lists, or reuse of any copyrighted component of this work in other works. This work has been accepted for publication in IEEE Robotics and Automation Letters.}}}}}}

\begin{document}

\title{Projection-Retraction MPPI: \\ Exact Command-Space Constraint-Manifold Control for Manipulators}
\author{Seulchan Lee\,\orcidlink{0009-0002-8330-2874}, Leesai Park\,\orcidlink{0009-0001-3572-9921}, Minhyeong Kang\,\orcidlink{0009-0004-7945-6357}, and Sanghyun Kim\,\orcidlink{0000-0002-0222-8688}
\thanks{Received 4 August 2026; accepted 17 November 2026. This paper was recommended for publication by Editor Lucia Pallottino upon evaluation of the Associate Editor and Reviewers' comments. This research was supported by the National Research Foundation of Korea (NRF) grant funded by the Korea government (MSIT) (No. RS-2024-00461583), by the Korea Basic Science Institute (National Research Facilities and Equipment Center) grant funded by the Ministry of Science and ICT (No. RS-2025-00564593), and by the “Research Program for Agriculture Science and Technology Development (No. RS-2026-25509723)” of the Rural Development Administration, Republic of Korea.}%
\thanks{S. Lee, L. Park, and M. Kang are with the Department of Mechanical Engineering, Kyung Hee University, Yongin 17104, South Korea (e-mail: lee081847@khu.ac.kr; leesai2000@khu.ac.kr; kingtyphoon@khu.ac.kr).}%
\thanks{S. Kim is with the Department of Mechanical Engineering, Kyung Hee University, Yongin 17104, South Korea, and also with MetaMotionX, Inc., Seongnam, South Korea (e-mail: kim87@khu.ac.kr).}%
}

\markboth{IEEE Robotics and Automation Letters. Preprint Version. Accepted November, 2026}%
{Lee \MakeLowercase{\textit{et al.}}: Projection-Retraction MPPI: Exact Command-Space Constraint-Manifold Control for Manipulators}

\IEEEpubid{\copyright~2026 IEEE}

\maketitle

\begin{abstract}
Model Predictive Path Integral (MPPI) control is widely used in manipulation
for its gradient-free, parallel handling of non-convex costs. Manipulation
tasks, however, often impose constraints that hold throughout the motion: a
closed kinematic chain that two grasping arms keep exactly, or joint limits
and obstacle clearances that are never crossed. Standard MPPI handles such
constraints only through the cost, as soft penalties that hold approximately
and fail under a strong task cost. To address this, we propose
Projection-Retraction MPPI (PR-MPPI), which enforces the constraints inside
the sampled dynamics. At every rollout step, the sampled velocity is projected to satisfy both
constraint types: the equality restricts it to a subspace, and each
inequality to a half-space within that subspace, so inequality handling
never breaks the equality. This projection, however, satisfies the
constraints only to first order, and a finite step leaves a small drift off
the equality. Therefore, we retract the returned command back onto the constraint, so that the commanded configuration satisfies the equality to numerical tolerance, independently of task weighting.
We validate PR-MPPI on 14-DoF dual-arm systems. 
In simulation, the returned commands satisfy the
closed-chain equality to numerical tolerance through a joint-limit stress
test and randomized obstacle avoidance, while the measured trajectories retain the joint-limit and obstacle margins. On real hardware, the arms of a
Unitree H1-2 humanoid reactively avoid a moving obstacle. Code and
experiment videos are available at
\url{https://rcilab.khu.ac.kr/prmppi/}.
\end{abstract}
 
\section{Introduction}
\label{sec:intro}

\IEEEPARstart{M}{odel} Predictive Path Integral (MPPI) control~\cite{williams2016aggressive, williams2017model} is widely used in manipulation. It optimizes the task by sampling control sequences and weighting them by their rollout costs, requiring no gradients, which suits the non-convex, non-smooth costs induced by contact and clutter. The rollouts are also evaluated fast through massive GPU parallelization, as demonstrated by frameworks such as STORM~\cite{bhardwaj2022storm}. Manipulation tasks, however, often impose constraints that hold throughout the motion, not merely at the goal. A wiping tool stays on the work surface. Two arms rigidly grasping an object form a closed kinematic chain whose relative pose holds exactly, or internal wrenches build up and the grasp is lost. All along the way, joint limits and obstacle clearances are never crossed. Standard MPPI handles constraints only through soft, penalty-based cost evaluation, and therefore cannot enforce a hard constraint.

To enforce constraints more strictly than a soft penalty, several lines of work augment MPPI with dedicated machinery on the inequality side. Chance-constrained and barrier-based formulations add probabilistic or forward-invariant safety~\cite{yin2024chance}. In the same spirit, discrete-time control-barrier filters such as Shield-MPPI~\cite{yin2023shield} repair the optimized control through a local barrier correction. A parallel line of work constrains the samples themselves rather than the cost. PRIEST~\cite{rastgar2023priest} projects them toward collision- and kinematic-feasible sets, and MPPI-Tan~\cite{zhao2025mppitan} redirects them along obstacle-gradient tangent spaces. CSC-MPPI~\cite{park2025cscmppi} pushes them out of obstacles by primal-dual gradient steps and selects the update from a DBSCAN cluster rather than a global average, while $\pi$-MPPI~\cite{andrejev2025pimppi} projects the sampled inputs onto magnitude and smoothness bounds. Whether by barrier repair or by sample correction, however, these methods consider one-sided \emph{inequality} margins only and offer no means of holding an equality.

\IEEEpubidadjcol
Enforcing a kinematic equality along a motion has a long history outside sampling-based MPC. Classical redundancy resolution holds task equalities instantaneously by projecting joint velocities onto the constraint null space~\cite{nakamura1986singularity, siciliano1991taskpriority, chan1995weighted, dariush2010cclik}. As a reactive one-step law, however, it performs no horizon optimization and offers no principled coupling to inequality margins. Sampling-based motion planners on constraint manifolds instead project or retract random configurations onto the manifold~\cite{berenson2009cbirrt, jaillet2013manifolds, kingston2019implicit}. This yields feasible paths offline, but not a closed-loop policy that re-solves the problem at every control cycle. These literatures thus supply the core operations, projection and retraction, but neither embeds them in a sampling-based MPC loop.

Two recent controllers bring this machinery into MPPI, and are closest to our setting. DQ-MPPI~\cite{zhu2026real} adopts the analytical route directly, projecting the samples onto the null space of the chain constraint. It handles the residual off-manifold drift only through null-space feedback at the low level, rather than removing it to a prescribed tolerance. The constraint is also confined to the dual-arm closed chain. MC-MPPI~\cite{lee2026mcmppi} instead replaces the analytic constraint with a feasible manifold learned by a variational autoencoder. It samples within that manifold and corrects the selected command with a single quadratic-program step, which restores feasibility only approximately. Because the feasible set is a learned model, its reliability also depends on the training distribution and model accuracy. Neither method, however, retracts the returned command until a prescribed equality tolerance is met, and both leave inequality constraints to the cost. As a result, neither returns an equality-feasible command to numerical tolerance while handling inequality constraints together, and neither keeps an inequality correction from pushing the state off the equality. Table~\ref{tab:related} summarizes this landscape.

\begin{figure*}[t] 
\centering 
\includegraphics[width=\textwidth]{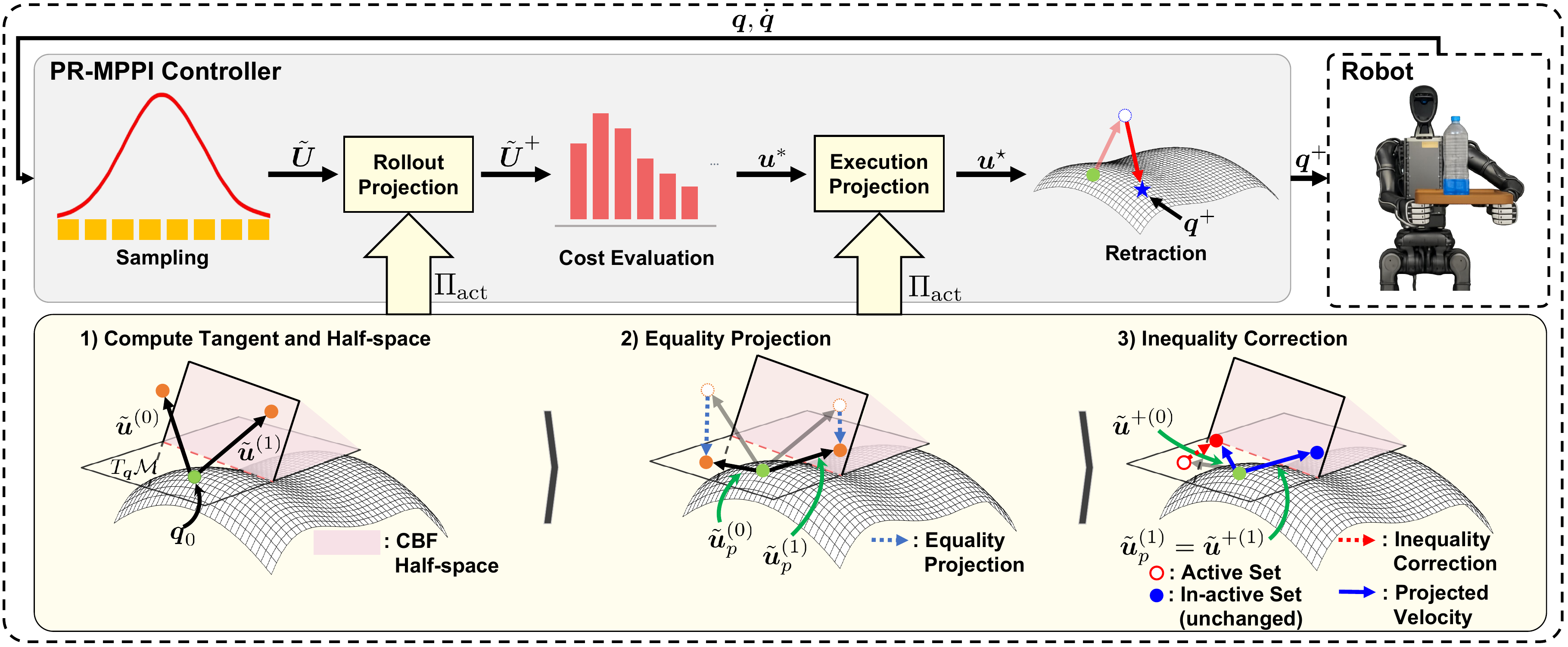} 
\caption{Overview of PR-MPPI. Sampled control sequences are rolled out
through the constraint-aware dynamics, where the active constraint
projector $\Pi_{\mathrm{act}}$ is applied at every step of every rollout.
The filtered inputs are averaged by the exponential-weighted update, the
averaged sequence is filtered once more along its own predicted states,
and the first step is retracted onto $\mathcal{M}$. The lower panels illustrate 
$\Pi_{\mathrm{act}}$ on two rollout samples in
three steps: (1) the tangent space of $\mathcal{M}$ and the CBF half-space
of an active margin, (2) projection of the sampled inputs onto the tangent
space, and (3) correction of the violating input within that tangent
space, leaving the admissible one unchanged.} 
\label{fig:overview} 
\end{figure*}

\begin{table}[t]
\centering
\footnotesize

\begin{threeparttable}
\caption{Constraint handling in sampling-based MPC.}
\label{tab:related}
\begin{tabularx}{\linewidth}{l Y Y Y Y}
\toprule
Method & Ineq.\tnote{a} & Eq.\tnote{a} & Tol.\tnote{b} & Free\tnote{c} \\
\midrule
Shield-MPPI~\cite{yin2023shield}     & \checkmark & $\times$   & $\times$   & \checkmark \\
PRIEST~\cite{rastgar2023priest}      & \checkmark & $\times$   & $\times$   & \checkmark \\
$\pi$-MPPI~\cite{andrejev2025pimppi} & \checkmark & $\times$   & $\times$   & \checkmark \\
MPPI-Tan~\cite{zhao2025mppitan}      & \checkmark & $\times$   & $\times$   & \checkmark \\
CSC-MPPI~\cite{park2025cscmppi}      & \checkmark & $\times$   & $\times$   & \checkmark \\
MC-MPPI~\cite{lee2026mcmppi}      & $\times$   & \checkmark & $\times$   & $\times$   \\
DQ-MPPI~\cite{zhu2026real}        & $\times$   & \checkmark & $\times$   & \checkmark   \\
\textbf{PR-MPPI (ours)}              & \checkmark & \checkmark & \checkmark & \checkmark \\
\bottomrule
\end{tabularx}
\begin{tablenotes}
\scriptsize
\item[a] The constraint is enforced within the sampled rollouts or by a
dedicated repair layer, rather than only through a cost penalty. Eq.\ denotes a
configuration-space equality maintained throughout the motion.
\item[b] The returned command satisfies the equality to numerical tolerance.
\item[c] Constraint satisfaction does not require a learned model.
\end{tablenotes}
\end{threeparttable}
\vspace{-0.5cm}
\end{table}

To bridge this gap, we propose Projection-Retraction MPPI (PR-MPPI), which enforces equality and inequality constraints within the rollout itself rather than through the cost. At every rollout step, each sampled velocity passes through a single velocity-space projection: the equality confines it to the tangent subspace of the constraint manifold, and every active inequality confines it further to a half-space inside that subspace. The two are therefore not combined by weighting or priority but nested: an inequality correction stays inside the equality subspace by construction, for any activation pattern. This matters because a naive fix, such as clamping a joint at its limit, can silently break the closed chain. The projection, however, holds the constraints only to first order, and a finite step leaves a small drift off the equality. A retraction step therefore pulls the commanded state back onto the manifold to the prescribed tolerance. Finally, the rollouts and the returned command pass through the same projection, so the plan is optimized under the same first-order constraint geometry that shapes the command.

The main contributions of this paper are as follows:


\begin{itemize}
\item We propose PR-MPPI, a constrained sampling-based MPC framework in
which a single velocity-space projection carries both constraint types
inside the sampled dynamics: each active inequality is corrected within the
equality tangent subspace, so inequality handling cannot break the equality
regardless of which margins are active. A retraction step then removes the
finite-step drift that the first-order projection leaves on the equality,
returning the command to satisfy the equality to numerical tolerance.

\item We establish a command-level equality property and characterize the
finite-step inequality error. Whenever the retraction terminates, the
returned command satisfies the equality constraint to the prescribed
numerical tolerance, independent of the MPPI weighting. Under the stated
assumptions, the per-step inequality error introduced by the finite step and the retraction is $O(\Delta t^2)$, and the resulting invariant layer can be absorbed by the safety margin.

\item We validate PR-MPPI on a $14$-DoF dual-arm Franka system in
simulation, where the retracted commands satisfy the closed-chain equality
to numerical tolerance while the measured trajectories retain the
joint-limit and obstacle margins under aggressive task commands, and
demonstrate reactive avoidance of a moving obstacle on the dual arms of a
Unitree H1-2 humanoid.
\end{itemize}

\section{Methodology}
\label{sec:methodology}

In this section, we present PR-MPPI, which incorporates equality and inequality constraints into MPPI through projection and retraction. During each sampled rollout, the perturbed velocity is projected onto the intersection of the equality-constraint tangent space and the active inequality half-spaces. After the MPPI update, the averaged sequence is filtered once more along its own predicted states, and the finite-step equality drift of the first command is removed by retracting the integrated configuration onto the constraint manifold. The overall framework is illustrated in Fig.~\ref{fig:overview}.

\subsection{Problem Definition}
\label{subsec:problem_definition}
The control input is the joint velocity, $\boldsymbol{u}_t=\dot{\boldsymbol{q}}_t\in\mathbb{R}^{n}$. We consider the discrete-time system
\begin{equation}
\boldsymbol{q}_{t+1} = f(\boldsymbol{q}_t, \tilde{\boldsymbol{u}}_t),
\label{eq:system_dynamics}
\end{equation}
where $\boldsymbol{q}_t \in \mathbb{R}^n$ is the joint configuration and $\tilde{\boldsymbol{u}}_t\in\mathbb{R}^n$ is the stochastic control input. Our objective is to find an optimal nominal control sequence $\boldsymbol{U}=[\boldsymbol{u}_0,\dots,\boldsymbol{u}_{T-1}]$ over a prediction horizon $T$ that minimizes the expected trajectory cost while satisfying both hard equality and inequality constraints under stochastic perturbations. The optimization problem is formulated as follows:
\begin{equation}
\begin{aligned}
    \min_{\boldsymbol{U}} \; & J = \mathbb{E}\left[ \phi(\boldsymbol{q}_{T}) + \sum_{t=0}^{T-1} \left( l(\boldsymbol{q}_t) + \frac{1}{2}\boldsymbol{u}_t^T \boldsymbol{R}\boldsymbol{u}_t \right) \right], \\
    \text{s.t.} \quad & \boldsymbol{c}(\boldsymbol{q}_t) = \boldsymbol{0}, \quad \boldsymbol{h}(\boldsymbol{q}_t) \geq \boldsymbol{0}, \\
    & \tilde{\boldsymbol{u}}_t = \boldsymbol{u}_t + \delta\boldsymbol{u}_t, \quad \delta\boldsymbol{u}_t \sim \mathcal{N}(\boldsymbol{0}, \boldsymbol{\Sigma}_u), \\
    & \boldsymbol{q}_0 = \boldsymbol{q}_{\mathrm{init}}, \quad
    \boldsymbol{q}_{t+1}=f(\boldsymbol{q}_t,\tilde{\boldsymbol{u}}_t),
\end{aligned}
\label{eq:problem_formulation}
\end{equation}
where $\phi(\boldsymbol{q}_T)$ and $l(\boldsymbol{q}_t)$ denote the terminal cost and running state cost, respectively. $\boldsymbol{c}(\boldsymbol{q}_t) \in \mathbb{R}^{n_e}$ and $\boldsymbol{h}(\boldsymbol{q}_t) \in \mathbb{R}^{n_i}$ denote the equality and inequality constraints imposed on the joint space, respectively. The matrix $\boldsymbol{R} \in \mathbb{R}^{n \times n}$ is a positive-definite weight matrix penalizing the nominal control effort, and $\boldsymbol{\Sigma}_u \in \mathbb{R}^{n\times n}$ is the exploration covariance.

\subsection{Constraint-Aware Sampled Dynamics}
\label{subsec:constraint_aware_sampled_dynamics}
The constrained problem in \eqref{eq:problem_formulation} requires the sampled rollouts to respect both equality constraints and inequality margins. A direct penalty-based treatment is undesirable because a sufficiently large task cost can still make sampled trajectories violate the hard constraints. Instead of treating these constraints only through the cost, PR-MPPI projects each sampled velocity in velocity space before it enters the dynamics. The trajectory costs are then evaluated on near-feasible motions, so the MPPI average is formed over rollouts that pass through the same projection as the final control input. We first express both constraint types as conditions on the velocity and then define the filter and the resulting sampled dynamics.

The equality constraint defines the manifold
\begin{equation}
    \mathcal{M}
    =
    \left\{
    \boldsymbol{q}
    :
    \boldsymbol{c}(\boldsymbol{q})=\boldsymbol{0}
    \right\},
    \qquad
    \boldsymbol{J}_c(\boldsymbol{q})
    =
    \frac{\partial \boldsymbol{c}}{\partial \boldsymbol{q}} \in \mathbb{R}^{n_e \times n},
    \label{eq:manifold_def}
\end{equation}
and, to first order, a velocity control input $\boldsymbol{u}$ preserves the constraint if and only if $\boldsymbol{J}_c(\boldsymbol{q})\boldsymbol{u}=\boldsymbol{0}$. For the inequalities, the desired safety buffer $\boldsymbol{h}_{\mathrm{safe}}\in\mathbb{R}^{n_i}$, elementwise positive, is embedded directly into the inequality function, and the shifted margins are kept nonnegative to first order by the CBF condition:
\begin{equation}
    \bar{\boldsymbol{h}}(\boldsymbol{q})
    =
    \boldsymbol{h}(\boldsymbol{q})
    -
    \boldsymbol{h}_{\mathrm{safe}},
    \qquad
    \boldsymbol{J}_h(\boldsymbol{q})\,
    \boldsymbol{u}
    \geq
    -\boldsymbol{\Gamma}\,
    \bar{\boldsymbol{h}}(\boldsymbol{q}),
    \label{eq:shifted_margin_cbf}
\end{equation}
where $\boldsymbol{J}_h(\boldsymbol{q}) =\partial\bar{\boldsymbol{h}}/\partial\boldsymbol{q} \in\mathbb{R}^{n_i\times n}$ and $\boldsymbol{\Gamma}=\mathrm{diag}(\gamma_1,\dots,\gamma_{n_i})$ with $\gamma_i>0$. The shifted guard condition is $\bar{\boldsymbol{h}}(\boldsymbol{q})\geq\boldsymbol{0}$, and each row of the CBF condition bounds how fast the velocity may approach the corresponding boundary. Both conditions are linear in the input. The active constraint projector $\Pi_{\mathrm{act}}$ updates the sampled input to its projection onto their intersection:
\begin{equation}
\begin{aligned}
    \tilde{\boldsymbol{u}}^{+}_{t}
    =
    \Pi_{\mathrm{act}}
    \left(
    \boldsymbol{q}_t,
    \tilde{\boldsymbol{u}}_t
    \right)
    =
    \arg\min_{\boldsymbol{u}}
    \;&
    \tfrac{1}{2}
    \left\|
    \boldsymbol{u}-\tilde{\boldsymbol{u}}_t
    \right\|^2
    \\
    \mathrm{s.t.}\quad
    &
    \boldsymbol{J}_c(\boldsymbol{q}_t)\boldsymbol{u}=\boldsymbol{0},
    \\
    &
    \boldsymbol{J}_h(\boldsymbol{q}_t)\boldsymbol{u}
    \geq
    -\boldsymbol{\Gamma}\,
    \bar{\boldsymbol{h}}(\boldsymbol{q}_t).
\end{aligned}
    \label{eq:active_cbf_projection}
\end{equation}
The system dynamics in \eqref{eq:system_dynamics} then take the constraint-aware form
\begin{equation}
\boldsymbol{q}_{t+1}
=
f_{\mathcal{M}}
\left(
\boldsymbol{q}_t,
\tilde{\boldsymbol{u}}_t
\right)
=
\boldsymbol{q}_{t}
+
\tilde{\boldsymbol{u}}^{+}_{t}\,\Delta t,
\label{eq:constraint_aware_dynamics}
\end{equation}
and every sampled rollout is generated by \eqref{eq:constraint_aware_dynamics}. The rest of this subsection describes how \eqref{eq:active_cbf_projection} is computed inside the rollouts, in two stages: an equality tangent-space projection and an active inequality correction.

The first stage projects the sampled input onto the tangent space of $\mathcal{M}$ through the null-space projector:
\begin{equation}
    \tilde{\boldsymbol{u}}_{p,t}
    =
    \boldsymbol{N}(\boldsymbol{q}_t)\,
    \tilde{\boldsymbol{u}}_t,
    \qquad
    \boldsymbol{N}(\boldsymbol{q})
    =
    \boldsymbol{I}
    -
    \boldsymbol{J}_c^{\dagger}(\boldsymbol{q})
    \boldsymbol{J}_c(\boldsymbol{q}),
    \label{eq:equality_projected_input}
\end{equation}
where $\boldsymbol{N}\in\mathbb{R}^{n\times n}$ is the null-space projector and $\boldsymbol{J}_c^{\dagger}$ the pseudoinverse of $\boldsymbol{J}_c$, so that $\boldsymbol{J}_c(\boldsymbol{q}_t)\tilde{\boldsymbol{u}}_{p,t}=\boldsymbol{0}$. The second stage enforces the CBF half-spaces within this tangent space. Since each sampled input violates only a few half-spaces, if any, solving \eqref{eq:active_cbf_projection} with all $n_i$ half-spaces at every step of every rollout is unnecessary. The projection is therefore solved over a small active set that is grown only when needed from the initial members selected by the CBF residual of each margin at the equality-projected input:
\begin{equation}
    r_i
    \left(
    \boldsymbol{q}_t,
    \boldsymbol{u}
    \right)
    =
    -\gamma_i
    \bar h_i(\boldsymbol{q}_t)
    -
    \nabla \bar h_i(\boldsymbol{q}_t)^{T}
    \boldsymbol{u},
    \label{eq:cbf_residual}
\end{equation}
where $\nabla\bar h_i(\boldsymbol{q})^{\top}$ is the $i$-th row of
$\boldsymbol{J}_h(\boldsymbol{q})$. The $i$-th CBF condition in
\eqref{eq:shifted_margin_cbf} holds if and only if
$r_i(\boldsymbol{q}_t,\boldsymbol{u})\le0$, whereas
$r_i(\boldsymbol{q}_t,\boldsymbol{u})>0$ means that the input decreases the
$i$-th margin faster than the CBF condition permits. The violated rows with
$r_i>0$ therefore form the initial active set. In addition, rows with
$\bar h_i(\boldsymbol{q}_t)<h_{\mathrm{band}}$ are included in advance, since
correcting other rows can push them into violation,
\begin{equation}
    \mathcal{A}^{(0)}_t
    =
    \left\{
    i:
    r_i(\boldsymbol{q}_t,\tilde{\boldsymbol{u}}_{p,t})>0
    \;\;\text{or}\;\;
    \bar h_i(\boldsymbol{q}_t)<h_{\mathrm{band}}
    \right\}.
    \label{eq:active_margin_selection}
\end{equation}
If $\mathcal{A}^{(0)}_t=\emptyset$, the equality-projected input already
satisfies every half-space, so
$\tilde{\boldsymbol{u}}^{+}_{t}=\tilde{\boldsymbol{u}}_{p,t}$. Otherwise, the inequality correction is computed iteratively. At iteration
$\ell$, restricting \eqref{eq:active_cbf_projection} to the half-spaces in
$\mathcal{A}^{(\ell)}_t$ gives the KKT solution
\begin{equation}
    \boldsymbol{u}^{(\ell)}_{t}
    =
    \tilde{\boldsymbol{u}}_{p,t}
    +
    \sum_{i\in\mathcal{A}^{(\ell)}_t}
    \mu_{i,t}\,
    \boldsymbol{N}(\boldsymbol{q}_t)
    \nabla\bar h_i(\boldsymbol{q}_t),
    \quad
    \mu_{i,t}\ge0,
    \label{eq:ineq_correction}
\end{equation}
where $\mu_{i,t}$ are the dual variables of the selected CBF half-spaces, obtained on a working set $\mathcal{W}\subseteq\mathcal{A}^{(\ell)}_t$ of rows imposed as equalities from
\begin{equation}
\boldsymbol{G}_{\mathcal{W}}\boldsymbol{N}\boldsymbol{G}_{\mathcal{W}}^{\!\top}
\boldsymbol{\mu}_{\mathcal{W}}
=
-\boldsymbol{\Gamma}_{\mathcal{W}}\bar{\boldsymbol{h}}_{\mathcal{W}}(\boldsymbol{q}_t)
-\boldsymbol{G}_{\mathcal{W}}\tilde{\boldsymbol{u}}_{p,t},
\label{eq:mu_system}
\end{equation}
with $\boldsymbol{G}_{\mathcal{W}}$ stacking their gradients. Within
$\mathcal{A}^{(\ell)}_t$, rows with $\mu_{i,t}<0$ violate the sign condition
in \eqref{eq:ineq_correction} and are dropped from $\mathcal{W}$, together
with rows with $\mu_{i,t}=0$, which do not affect the solution, while rows
with $r_i>0$ are added, up to a fixed iteration cap.
Outside $\mathcal{A}^{(\ell)}_t$, if
$r_i(\boldsymbol{q}_t,\boldsymbol{u}^{(\ell)}_{t})>0$ for some $i$, the row
with the largest $r_i$ is added,
\begin{equation}
    \mathcal{A}^{(\ell+1)}_t
    =
    \mathcal{A}^{(\ell)}_t
    \cup
    \Bigl\{
    \arg\max_{i\notin\mathcal{A}^{(\ell)}_t}
    r_i(\boldsymbol{q}_t,\boldsymbol{u}^{(\ell)}_{t})
    \Bigr\},
    \label{eq:active_set_update}
\end{equation}
and the restricted solve is repeated. The iteration terminates when every omitted residual satisfies $r_i(\boldsymbol{q}_t,\boldsymbol{u}^{(\ell)}_{t})\le0$, and the final iterate $\tilde{\boldsymbol{u}}^{+}_{t}$ is the minimizer of \eqref{eq:active_cbf_projection}, provided that the KKT conditions are satisfied. Since every correction in \eqref{eq:ineq_correction} is projected by $\boldsymbol{N}$ onto the tangent space of $\mathcal{M}$, $\boldsymbol{J}_c\tilde{\boldsymbol{u}}^{+}_{t}=\boldsymbol{0}$. If the projection fails the final feasibility check, which tests the returned input against every inequality half-space, the rollout step proceeds with the equality-projected input $\tilde{\boldsymbol{u}}_{p,t}$, which still satisfies $\boldsymbol{J}_c\tilde{\boldsymbol{u}}_{p,t}=\boldsymbol{0}$.

Each rollout generated by \eqref{eq:constraint_aware_dynamics} is assigned the trajectory cost of \eqref{eq:problem_formulation}, yielding $\{S^{(k)}\}_{k=1}^{K}$. Let $\tilde{\boldsymbol{U}}^{+(k)} =[\tilde{\boldsymbol{u}}^{+(k)}_0,\dots,\tilde{\boldsymbol{u}}^{+(k)}_{T-1}]$ denote the filtered input sequence of the $k$-th rollout. The nominal sequence is then updated by exponential-weighted averaging of the filtered sequences:
\begin{equation}
    \boldsymbol{U}^{*}
    =
    \sum_{k=1}^{K} w_k\,
    \tilde{\boldsymbol{U}}^{+(k)},
    \qquad
    w_k
    \propto
    \exp\!\left(-\tfrac{1}{\lambda}S^{(k)}\right),
    \label{eq:mppi_update}
\end{equation}
where $\lambda>0$ is the temperature parameter and the weights are normalized to sum to one.

\subsection{Execution-Layer Projection and Retraction}
\label{subsec:exec_projection_retraction}
The MPPI update \eqref{eq:mppi_update} returns
$\boldsymbol{U}^{*}=[\boldsymbol{u}^{*}_0,\dots,\boldsymbol{u}^{*}_{T-1}]$ by
exponential-weighted averaging of the projected inputs. Each sample, however,
was projected at the states of its own rollout, so the averaged inputs carry no
constraint guarantee along the trajectory that $\boldsymbol{U}^{*}$ itself
generates. The averaged sequence therefore requires filtering before actuation,
and the system dynamics \eqref{eq:constraint_aware_dynamics} serve this role.
Starting from the measured state $\boldsymbol{q}_0$, the sequence is rolled
forward, and each input is filtered into the final control input for
$t=0,\dots,T-1$:
\begin{equation}
    \boldsymbol{u}^{\star}_t
    =
    \Pi_{\mathrm{act}}
    \left(
    \boldsymbol{q}_t,
    \boldsymbol{u}^{*}_t
    \right).
    \label{eq:exec_projection}
\end{equation}
This execution filtering corrects each input on the state it will actually visit, with the same projector as in the rollouts. The first input is a special case. All samples share $\boldsymbol{q}_0$
at $t=0$, so if every first-step rollout projection passes the feasibility
check, the filtered inputs lie in the same convex feasible set
of~\eqref{eq:active_cbf_projection}, and their weighted average, being a
convex combination, remains within it. The filtering then leaves
$\boldsymbol{u}^{*}_0$ unchanged, and its corrections act on the inputs for
$t\ge1$. The initial execution active set follows the same rule,
\begin{equation}
    \mathcal{A}^{(0)}_t
    =
    \bigl\{
    i
    :
    r_i(\boldsymbol{q}_t,\boldsymbol{N}\boldsymbol{u}^{*}_t)>0
    \ \text{or}\ 
    \bar h_i(\boldsymbol{q}_t)<h_{\mathrm{band}}
    \bigr\},
    \label{eq:exec_active_set}
\end{equation}
so simultaneously critical margins are enforced jointly, and the active-set iteration \eqref{eq:active_set_update} refines this set until every omitted half-space is satisfied. Before actuation, feasibility is checked once more at the returned
input, and if the check fails, the robot holds its last retracted command
without relaxing any inequality. Every actuated command thus passes this
execution-layer check, even if every rollout becomes infeasible.

Although \eqref{eq:exec_projection} enforces $\boldsymbol{J}_c(\boldsymbol{q}_0)\boldsymbol{u}^{\star}_0=\boldsymbol{0}$, this is a first-order condition on the commanded input. Starting from $\boldsymbol{q}_0\in\mathcal{M}$, the finite step therefore leaves $\boldsymbol{c}(\boldsymbol{q}_0+\Delta t\,\boldsymbol{u}^{\star}_0) =O(\Delta t^{2})$. We remove this finite-step drift by initializing $\boldsymbol{q}=\boldsymbol{q}_0+\Delta t\,\boldsymbol{u}^{\star}_0$ and applying the Gauss--Newton iteration
\begin{equation}
    \boldsymbol{q}
    \leftarrow
    \boldsymbol{q}
    -
    \boldsymbol{J}_c^{\dagger}(\boldsymbol{q})\,
    \boldsymbol{c}(\boldsymbol{q})
    \qquad
    \text{until}
    \quad
    \left\|\boldsymbol{c}(\boldsymbol{q})\right\| < \varepsilon_{\mathrm{tol}}.
    \label{eq:exec_retraction}
\end{equation}
The iteration defines the retraction map $\mathcal{R}_{\mathcal{M}}$, yielding the final command $\boldsymbol{q}^{+}=\mathcal{R}_{\mathcal{M}} (\boldsymbol{q}_0+\Delta t\,\boldsymbol{u}^{\star}_0)$.

\begin{algorithm}[t]
\caption{Projection-Retraction MPPI Framework}
\label{alg:pr_mppi}
\begin{algorithmic}[1]
\Require
\State \hspace{1em} $\boldsymbol{q}_0$: Current measured configuration
\State \hspace{1em} $\boldsymbol{U}=[\boldsymbol{u}_0,\dots,\boldsymbol{u}_{T-1}]$: Nominal control sequence
\State \hspace{1em} $\boldsymbol{\Sigma}_u$: Exploration covariance
\State \hspace{1em} $K,T$: Number of samples and horizon
\State \hspace{1em} $\boldsymbol{c},\bar{\boldsymbol{h}}$: Equality and shifted inequality constraints
\For{$k=1,\ldots,K$} \Comment{Parallel sampling}
    \State $\delta\boldsymbol{u}_{0:T-1}^{(k)}\sim\mathcal{N}(\boldsymbol{0},\boldsymbol{\Sigma}_u)$
    \State $\boldsymbol{q}^{(k)}_0\leftarrow\boldsymbol{q}_0$
    \For{$t=0,\ldots,T-1$}
        \State $\tilde{\boldsymbol{u}}_t^{(k)}\leftarrow\boldsymbol{u}_t+\delta\boldsymbol{u}_t^{(k)}$
        \State $\tilde{\boldsymbol{u}}_{t}^{+(k)}\leftarrow\Pi_{\mathrm{act}}(\boldsymbol{q}^{(k)}_t,\tilde{\boldsymbol{u}}_t^{(k)})$ \Comment{Eq.~\eqref{eq:active_cbf_projection}}
        \State $\boldsymbol{q}^{(k)}_{t+1}\leftarrow\boldsymbol{q}^{(k)}_t+\tilde{\boldsymbol{u}}_{t}^{+(k)} \Delta t$ \Comment{Eq.~\eqref{eq:constraint_aware_dynamics}}
    \EndFor
    \State Evaluate trajectory cost $S^{(k)}$
\EndFor
\State Compute exponential weights $\{w^{(k)}\}_{k=1}^{K}$ via \eqref{eq:mppi_update}
\State $\boldsymbol{U}^{*}\leftarrow\sum_{k=1}^{K}w^{(k)}\tilde{\boldsymbol{U}}^{+(k)}$ \Comment{MPPI update}
\For{$t=0,\ldots,T-1$}
    \State $\boldsymbol{u}^{\star}_t\leftarrow\Pi_{\mathrm{act}}(\boldsymbol{q}_t,\boldsymbol{u}^{*}_t)$ \Comment{Execution filtering, Eq.~\eqref{eq:exec_projection}}
    \State $\boldsymbol{q}_{t+1}\leftarrow\boldsymbol{q}_t+\boldsymbol{u}^{\star}_t\,\Delta t$
\EndFor
\State $\boldsymbol{q}^{+}\leftarrow\mathcal{R}_{\mathcal{M}}(\boldsymbol{q}_0+\Delta t\,\boldsymbol{u}^{\star}_0)$ \Comment{Retraction}
\State $\boldsymbol{U}\leftarrow[\boldsymbol{u}^{\star}_{1},\ldots,\boldsymbol{u}^{\star}_{T-1},\boldsymbol{0}]$ \Comment{Warm start}
\State \Return $\boldsymbol{q}^{+}$
\end{algorithmic}
\end{algorithm}

\begin{remark}[Command-space exactness]
\label{rem:retraction_exactness}
In this paper, command-space exactness refers to numerical satisfaction of the equality constraint,
$\|\boldsymbol{c}(\boldsymbol{q}^{+})\|<\varepsilon_{\mathrm{tol}}$,
upon successful termination of the retraction in~\eqref{eq:exec_retraction}, independently of the task cost, the MPPI temperature and covariance, and the inequality-handling parameters. It applies to the returned command rather than to the measured robot state, whose equality residuals additionally reflect low-level tracking error.
\end{remark}

\begin{lemma}[Finite-step margin bound]
\label{lem:finite_step_margin_bound}
Let $\boldsymbol{c}$ and each $\bar h_i$ be twice continuously
differentiable with bounded second derivatives on the operating set, on
which $\boldsymbol{J}_c$ has constant row rank and
$\|\boldsymbol{J}_c^{\dagger}\|$ is bounded. Let $\boldsymbol{q}_t$ and
$\boldsymbol{u}^{\star}_t$ denote the configuration and the actuated input at
cycle $t$. Suppose that, at every cycle, the retraction
\eqref{eq:exec_retraction} terminates with
$\|\boldsymbol{c}(\boldsymbol{q}_t)\|<\varepsilon_{\mathrm{tol}}=O(\Delta t^{2})$,
\eqref{eq:active_cbf_projection} is feasible at $\boldsymbol{q}_t$,
$\|\boldsymbol{u}^{\star}_t\|\leq\bar u$,
$\|\nabla\bar h_i\|\leq\bar L$, and $\gamma_i\Delta t\leq1$. Then there exist
constants $\kappa_i\geq0$, independent of $\Delta t$, such that
\begin{equation}
\bar h_i(\boldsymbol{q}_{t+1})
\geq
(1-\gamma_i\Delta t)\bar h_i(\boldsymbol{q}_t)
-\kappa_i\Delta t^2 ,
\label{eq:finite_step_margin_bound}
\end{equation}
so $\bar h_i(\boldsymbol{q}_t)\geq0$ implies
$\bar h_i(\boldsymbol{q}_{t+1})\geq-\kappa_i\Delta t^2$.
\end{lemma}
\begin{proof}
Since $\|\boldsymbol{c}(\boldsymbol{q}_t)\|<\varepsilon_{\mathrm{tol}}$ and
$\boldsymbol{J}_c(\boldsymbol{q}_t)\boldsymbol{u}^{\star}_t=\boldsymbol{0}$,
the integrated step has an $O(\Delta t^{2})$ equality residual, and the
rank and pseudoinverse bounds make the retraction displacement of the same
order. The actuated input satisfies every CBF half-space,
$\nabla\bar h_i(\boldsymbol{q}_t)^{\top}\boldsymbol{u}^{\star}_t
\geq-\gamma_i\bar h_i(\boldsymbol{q}_t)$, so a Taylor expansion of
$\bar h_i$ along the step, with the second-order and retraction terms
absorbed into $\kappa_i$, gives \eqref{eq:finite_step_margin_bound}. The
implication follows from $\gamma_i\Delta t\leq1$.
\end{proof}
The lower-bound recursion in~\eqref{eq:finite_step_margin_bound} has the fixed point
\begin{equation}
\bar h_i
=
-\frac{\kappa_i}{\gamma_i}\Delta t .
\label{eq:finite_step_fixed_point}
\end{equation}
Hence, for $\gamma_i\Delta t\leq1$, the bound $\bar h_i(\boldsymbol{q})\geq-(\kappa_i/\gamma_i)\Delta t$ is forward invariant. Therefore, if $\bar h_i(\boldsymbol{q}_0)\geq0$, the sufficient buffer condition
\begin{equation}
h_{\mathrm{safe},i}
\geq
\frac{\kappa_i}{\gamma_i}\Delta t
\label{eq:safety_buffer_condition}
\end{equation}
ensures $h_i(\boldsymbol{q}_t)\geq0$ at every cycle. For affine margins such as joint limits, the Euler-step remainder vanishes, leaving only the retraction displacement $\boldsymbol{\delta}_t$. The corresponding sufficient buffer is
\begin{equation}
h_{\mathrm{safe},i}
\geq
\frac{\sup_t\|\boldsymbol{\delta}_t\|}
{\gamma_i\Delta t}.
\label{eq:affine_buffer}
\end{equation}
With $\bar\gamma_i=\gamma_i\Delta t$ held fixed, the layer width becomes $(\kappa_i/\bar\gamma_i)\Delta t^2$, recovering the $O(\Delta t^2)$ finite-step layer. 

The two constraint types are handled differently by design.
Beyond the shifted guard, the CBF condition in~\eqref{eq:shifted_margin_cbf}
drives the state back toward the safe region, and
Lemma~\ref{lem:finite_step_margin_bound} bounds the finite-step deviation so
that a sufficiently large safety buffer can absorb it.
The equality condition $\boldsymbol{J}_c\boldsymbol{u}=\boldsymbol{0}$, in contrast, cannot correct
accumulated drift and admits no interior margin to absorb it, so the
retraction in~\eqref{eq:exec_retraction} restores the returned command to the
prescribed numerical tolerance.

\section{Experiments}
\label{sec:experiments}

\subsection{Experimental Setup}
\label{subsec:experimental_setup}

In this section, we present three experiments designed to evaluate the proposed method. Experiment~1 examines the multi-inequality projection when an aggressive joint-tracking objective activates multiple joint-limit margins simultaneously, and assesses the respective roles of inequality projection and equality retraction. Experiment~2 investigates how incorporating inequality projection into the sampled dynamics affects the generation of near-feasible rollout trajectories in a randomized obstacle-avoidance task. Experiment~3 evaluates hardware applicability and reactive avoidance of a moving obstacle on a humanoid robot. Experiments~1 and~2 use a dual-arm Franka Emika Panda in MuJoCo~\cite{todorov2012mujoco}, whereas Experiment~3 uses the two $7$-DoF arms of a Unitree H1-2 humanoid.

In every experiment, the arms rigidly grasp a tray with both end-effectors, and the enforced equality constraint is the resulting closed-chain condition
\begin{equation}
\boldsymbol{c}(\boldsymbol{q})
=
\begin{bmatrix}
\log\bigl(\boldsymbol{E}(\boldsymbol{q})\bigr)^{\vee}\\
\varphi_{o}\\
\vartheta_{o}
\end{bmatrix}
\in\mathbb{R}^{8},
\qquad
\boldsymbol{E}(\boldsymbol{q})
=
\boldsymbol{T}_{l}\,{}^{l}\boldsymbol{T}_{r}\,\boldsymbol{T}_{r}^{-1},
\label{eq:chain_constraint}
\end{equation}
where $\boldsymbol{T}_{l},\boldsymbol{T}_{r}\in SE(3)$ are the end-effector poses, ${}^{l}\boldsymbol{T}_{r}$ and ${}^{l}\boldsymbol{T}_{o}$ are the grasp-fixed relative poses of the right gripper and the tray, and $(\varphi_{o},\vartheta_{o})$ are the roll and pitch of the tray pose $\boldsymbol{T}_{l}\,{}^{l}\boldsymbol{T}_{o}$. The inequalities $\boldsymbol{h}(\boldsymbol{q})$ collect the joint-limit and obstacle margins specified in each scenario. All experiments run on an Intel Core i5-13400F CPU with 16~GB of RAM and an NVIDIA GeForce RTX 4060 Ti GPU (8~GB VRAM).

Table~\ref{tab:params} lists the parameters of PR-MPPI: the number of rollouts $K$, the horizon $T$, the time step $\Delta t$, the exploration standard deviation $\sigma$ with $\Sigma_u = \sigma^2 I$, the class-$\mathcal{K}$ gain $\gamma_i$, and the retraction tolerance $\varepsilon_{\mathrm{tol}}$. PR-MPPI replans every $\Delta t$, and a 500 Hz computed-torque law tracks the retracted command. Task costs and remaining parameters are given on the project page, \url{https://rcilab.khu.ac.kr/prmppi/}.

\begin{table}[t]
\caption{Parameter settings of PR-MPPI}
\label{tab:params}
\centering
\begin{tabular}{lcccccc}
\toprule
 & $K$ & $T$ & $\Delta t$ (s) & $\sigma$ (rad/s) & $\gamma_i$ & $\varepsilon_{\mathrm{tol}}$ \\
\midrule
Exp.~1, 2 & 1000 & 30 & $1/30$ & 0.03 & 5  & $10^{-6}$ \\
Exp.~3    & 200  & 30 & $1/20$ & 0.06 & 10 & $10^{-6}$ \\
\bottomrule
\end{tabular}
\end{table}


\begin{table*}[t]
\centering
\scriptsize
\begin{threeparttable}
\caption{Joint-limit violations, measured constraint residuals, and MPPI-inner
computation times over the $0$--$6$~$s$ interval.}
\label{tab:exp1_constraint}
\setlength{\tabcolsep}{2pt}
\begin{tabularx}{\textwidth}{lYYYYYY}
\toprule
Method
& \shortstack{Chain trans.\\($m$)}
& \shortstack{Chain rot.\\($rad$)}
& \shortstack{Tilt\\($rad$)}
& \shortstack{Max. bound\\viol. ($rad$)}
& \shortstack{Max. margin\\pen. ($rad$)}
& \shortstack{Compute time\\($ms$)} \\
\midrule
DQ-MPPI~\cite{zhu2026real}
& $0.004\!\pm\!0.000$
& $0.004\!\pm\!0.001$
& $0.061\!\pm\!0.047$
& $0.082$
& $0.132$
& $238.948\!\pm\!19.473$ \\
\addlinespace[2pt]
\textbf{PR-MPPI Full}\tnote{a}
& $0.003\!\pm\!0.000$
& $0.003\!\pm\!0.000$
& $0.002\!\pm\!0.000$
& $0.000$
& $0.003$
& $19.351\!\pm\!1.216$ \\
\addlinespace[2pt]
PR-MPPI w/o retraction
& $0.007\!\pm\!0.002$
& $0.008\!\pm\!0.003$
& $0.011\!\pm\!0.004$
& $0.000$
& $<0.001$
& $18.288\!\pm\!0.514$ \\
\addlinespace[2pt]
PR-MPPI w/o inequality\tnote{a}
& $0.003\!\pm\!0.000$
& $0.003\!\pm\!0.000$
& $0.002\!\pm\!0.000$
& $0.022$
& $0.072$
& $8.975\!\pm\!0.092$ \\
\bottomrule
\end{tabularx}
\begin{tablenotes}
\scriptsize
\item[a] The returned command satisfies every equality channel to below
$10^{-9}$ throughout. The table entries are measured-state residuals.
\end{tablenotes}
\end{threeparttable}
\end{table*}

\subsection{Experiment 1: Joint-Limit Stress Test}
\label{subsec:exp1}

The system is initialized at a configuration corresponding to a level tray at $(0.45, 0, 1.0)$~$m$ and tracks an IK solution for $(0.45, 0, 0.55)$~$m$. Both arms start from $q_4=-2.01$~$rad$, whereas the IK target is $q_{4,g}=-2.41$~$rad$. We impose the artificial lower bound $q_{4,\min}=-2.35$~$rad$ and a $0.05$~$rad$ CBF safety margin, which shifts the guard boundary to $-2.30$~$rad$. Because the target lies $0.06$~$rad$ beyond the bound in both arms, the two joint-limit constraints become active nearly simultaneously.

We compare Full PR-MPPI with DQ-MPPI and two ablations: PR-MPPI without retraction, which bypasses the final equality retraction, and PR-MPPI without inequality projection, which disables joint-limit projection in both the sampled dynamics and the execution layer. In this experiment, the joint-tracking cost weight was set to ten times the joint-limit violation penalty weight. No joint-limit violation penalty was applied to Full PR-MPPI or PR-MPPI without retraction.

Fig.~\ref{fig:exp1_joint_limits} and Table~\ref{tab:exp1_constraint} summarize the outcome. Because the tracking objective aggressively drives both fourth joints toward the bound, a soft joint-limit penalty can be traded against a sufficiently aggressive task objective. A cost-based treatment alone therefore cannot guarantee the limit, and the measurements confirm it: PR-MPPI without inequality projection and DQ-MPPI penetrate the bound by $0.022$ and $0.082$~$rad$, respectively. Full PR-MPPI uses no joint-limit violation penalty and instead removes the bound-violating component from every sampled and executed velocity: both joints settle onto the guard boundary and slide along it, the bound is never crossed, and the residual $0.003$~$rad$ margin penetration is absorbed by the $0.05$~$rad$ safety shift. Since the joint-limit margins are affine, the only finite-step effect is
the retraction displacement, whose maximum observed value of $0.003$~$rad$
yields a sufficient buffer of $0.018$~$rad$ by~\eqref{eq:affine_buffer}.
The selected $0.05$~$rad$ buffer therefore exceeds this value by a factor of
approximately $2.8$, and the measured states remained above $q_{4,\min}$
throughout the experiment, including the effect of low-level tracking errors
outside the scope of Lemma~\ref{lem:finite_step_margin_bound}.
Moreover, although the measured state entered the safety buffer in $15.7\%$
of the execution cycles, no QP infeasibility occurred in either the rollouts
or the execution layer.

\begin{figure}[t]
\centering
\subfloat[]{\includegraphics[width=0.98\columnwidth]{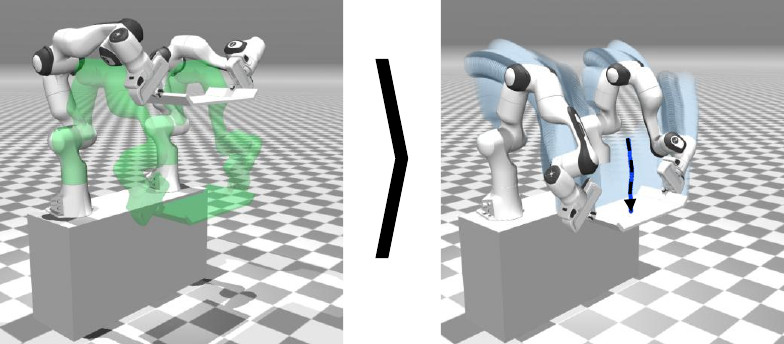}
\label{fig:exp1_motion}}\par
\vspace{-0.5pt}
\subfloat[]{\includegraphics[width=0.98\columnwidth]{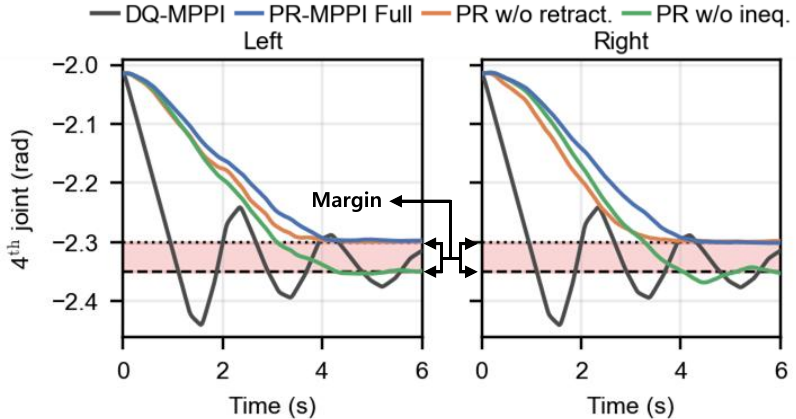}
\label{fig:exp1_joint_trajectories}}
\caption{Joint-limit stress test. \textbf{(a)}~Dual-arm tray-lowering motion:
the goal joint configuration in green and the trajectory executed by Full
PR-MPPI in blue. \textbf{(b)}~Measured fourth-joint trajectories of both arms for the
four controllers. The dashed line is the artificial lower bound
$q_{4,\min}=-2.35$~$rad$ and the dotted line the shifted guard boundary at
$-2.30$~$rad$; the shaded band between them is the $0.05$~$rad$ CBF margin.}
\label{fig:exp1_joint_limits}
\end{figure}

As shown in Table~\ref{tab:exp1_constraint}, the measured equality chain residuals
of all four methods remain at a similar $10^{-3}$ level, which reflects
low-level tracking error. In command space, however, both PR-MPPI variants with retraction satisfy every channel of \eqref{eq:chain_constraint} to the order of
$10^{-11}$, namely $(0.170\!\pm\!0.070)\!\times\!10^{-10}$~$m$ in chain
translation, $(0.280\!\pm\!0.020)\!\times\!10^{-10}$~$rad$ in chain rotation,
and $(0.055\!\pm\!0.052)\!\times\!10^{-10}$~$rad$ in tilt. The retraction
removes the finite-step drift of tangent-space integration from the command,
whereas DQ-MPPI and PR-MPPI without retraction correct it only through
null-space projection, so the drift remains in their commands. For tray tilt, DQ-MPPI handles the constraint only as a soft cost, which
results in a larger tilt residual.

As shown in Table~\ref{tab:exp1_constraint}, the computational differences reflect the structure of each method. PR-MPPI applies the inequality projection at every step of every sampled rollout, so its cost increases with the number of inequality constraints, while the retraction is applied only once to the selected command and contributes little overhead. DQ-MPPI, evaluated with the authors' official implementation\footnote{Available at \url{https://dq-gpu-mpc.github.io/}.}, performs two MPPI stages together with a clustering step. On the same hardware and at the same sampling budget, its planning update takes $238.948$~$ms$ against $19.351$~$ms$ for Full PR-MPPI. Both methods run their low-level null-space controller at $500$~Hz between planning updates; the difference lies in the replanning rate, about $4.2$~Hz for DQ-MPPI versus the sustained $30$~Hz of PR-MPPI.

\subsection{Experiment 2: Randomized Obstacle Avoidance}
\label{subsec:exp2}

The level tray moves from $(0.5,0,0.5)$~$m$ to $(0.5,0,1.0)$~$m$ past a spherical obstacle of radius $0.05$~$m$ centered at $(x_i,y_i,0.75)$~$m$, where $x_i$ and the lateral offset $y_i$ are drawn at random within $[0.40,0.60]$~$m$ and $[-0.10,0.10]$~$m$, respectively, for each of the $30$ paired trials and shared by all controllers, so that each must find a lateral detour around the blocked straight path in Fig.~\ref{fig:exp2_obstacle_task}. All methods use task-space tracking. PR-MPPI takes $h(\boldsymbol{q})$ as the tray--obstacle distance with safety margin $h_{\mathrm{safe}}=0.02$~$m$, whereas DQ-MPPI and MC-MPPI~\cite{lee2026mcmppi} rely on a soft collision cost alone. Full PR-MPPI handles the obstacle through the projection alone, without any
collision penalty. An execution-only ablation applies the projection only at
the execution layer and instead carries the same squared penetration penalty
as the baselines, with all other settings identical.

All controllers run for $120$~$s$ against the same outcome geometry: a collision is declared when the tray contacts the obstacle sphere, which exerts no physical force, and a run succeeds when the tray pose error $e=\|\log(\boldsymbol{T}^{-1}\boldsymbol{T}_{g})^{\vee}\|_2$, which stacks the translational ($m$) and rotational ($rad$) components of the relative twist, stays below $0.02$ \footnote{Matching the criterion of~\cite{zhu2026real}, which thresholds the dual-quaternion pose-error norm---approximately $e/2$ for small errors---at $0.01$.} for $1$~$s$; failed runs are classified by their first failure event as collision, tray drop, or stuck.

Table~\ref{tab:exp2_randomized_obstacle} summarizes the outcomes, which separate primarily by how each method handles the obstacle inequality. MC-MPPI completes $13$ of $30$ trials, while all $17$ unsuccessful trials are classified as stuck: no collision or tray drop occurs, but the controller does not attain the goal dwell within the time limit. The stuck behavior is most evident when the obstacle lies near the center of the nominal path, around $(0.5,0,0.75)$~$m$. Because the tray is wide relative to the available clearance, a large fraction of the sampled trajectories then intersect the obstacle, leaving too few informative feasible samples that pass around either side. The cost-only update consequently fails to establish a consistent lateral detour, and the motion stalls near the obstacle.

DQ-MPPI uses a two-stage update in which the first stage performs
high-variance exploration, which is reflected in larger control deviations
and longer executed paths. In contrast, PR-MPPI projects the sampled inputs into the
constraint-admissible space, allowing the MPPI update to use feasible
rollouts to directly search for an effective path around the obstacle. Thus, over the same $18$ randomized obstacle placements where DQ-MPPI
succeeds, its mean path length is $1.338$~$m$, compared with $0.616$~$m$ for
Full PR-MPPI. The larger step-to-step motion also magnifies the finite-step
error of the first-order constraint correction, increasing the closed-chain
residuals. Indeed, the tray-drop trials exhibit larger mean chain translation,
chain rotation, and tilt residuals than the successful trials: the respective
values increase from $0.0111$~$m$, $0.0105$~$rad$, and $0.0239$~$rad$ to
$0.0159$~$m$, $0.0181$~$rad$, and $0.0301$~$rad$.
The resulting loss of closed-chain consistency can degrade tray support.
Of the $12$ failures, $11$ are tray drops and one is stuck.
To assess whether the lower replanning rate accounts for these
differences, we additionally evaluated PR-MPPI at $4.2$~Hz while keeping
the remaining settings unchanged. PR-MPPI still completed all $30$ trials,
indicating that its inequality projection within the sampled dynamics can
generate informative feasible rollouts that form lateral detours around the
obstacle even at the matched replanning rate.

\begin{table}[t]
\centering
\scriptsize
\setlength{\tabcolsep}{2pt}
\caption{Outcome counts over $30$ randomized obstacle placements.}
\label{tab:exp2_randomized_obstacle}
\begin{tabularx}{\columnwidth}{lYYYY}
\toprule
Method
& \shortstack{Success/\\Total}
& Stuck
& \shortstack{Tray\\drop}
& Collision \\
\midrule
MC-MPPI~\cite{lee2026mcmppi}
& $13/30$
& $17$
& $0$
& $0$ \\
DQ-MPPI~\cite{zhu2026real}
& $18/30$
& $1$
& $11$
& $0$ \\
\textbf{PR-MPPI Full}
& $30/30$
& $0$
& $0$
& $0$ \\
PR-MPPI Full @4.2Hz
& $30/30$
& $0$
& $0$
& $0$ \\
\shortstack[l]{PR-MPPI Exec.-Only Ineq.}
& $29/30$
& $1$
& $0$
& $0$ \\
\bottomrule
\end{tabularx}
\end{table}

\begin{figure}[t]
\centering
\includegraphics[width=\columnwidth]{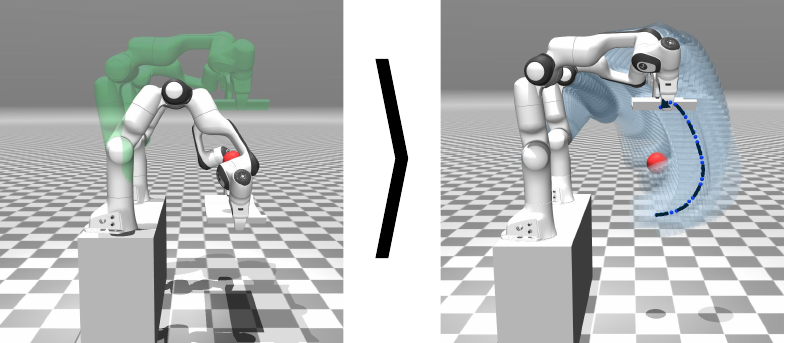}
\caption{Randomized obstacle-avoidance task at the central lateral offset. The
tray is transported past the spherical obstacle (red) under the closed-chain
grasp; green marks the goal tray pose and blue a representative collision-free
detour by Full PR-MPPI, with its dashed tray-center path.}
\label{fig:exp2_obstacle_task}
\end{figure}

The execution-only ablation separates the two roles of the projection. With the execution filtering kept, the constraints themselves hold: whatever the update proposes, the actuated command is corrected on the states it visits, and no collision or tray drop occurs in any trial. The plan, however, is still evaluated on rollouts that violate the obstacle margin, and the cost of this appears in how the search behaves. When the obstacle blocks the center of the nominal path, few informative samples pass around either side, so the search takes longer or stalls. The ablation reaches the goal dwell in $18.217\pm10.183$~$s$ against $17.951\pm5.046$~$s$ for Full PR-MPPI, with twice the spread, and the same effect appears as path variability: it travels $0.757\pm0.107$~$m$ against $0.636\pm0.040$~$m$, with the largest detours at placements close to the nominal path. Its single failure, out of $30$ trials, is stuck at exactly such a placement. Full PR-MPPI projects every sampled input inside the rollouts without any QP infeasibility, so the trajectory cost is evaluated on motions that respect the margin and the update averages over rollouts that already pass the obstacle. It runs with the tray--obstacle penalty disabled entirely and completes every trial.

\begin{figure*}[t]
\centering
\includegraphics[width=\textwidth]{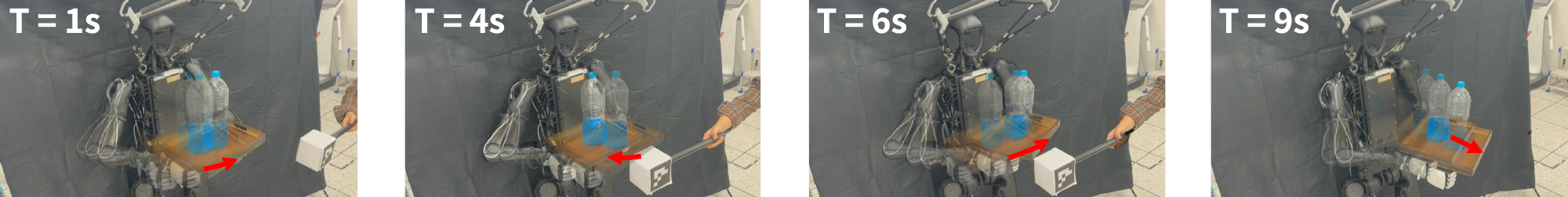}
\caption{Hardware demonstration on the Unitree H1-2. The two arms transport
a loaded tray under a closed-chain grasp while a hand-held obstacle is moved
toward the tray. Each frame shows the elapsed time.}
\label{fig:exp3_h12_dynamic_obstacle}
\end{figure*}
\subsection{Experiment 3: Reactive Avoidance of a Moving Obstacle on Hardware}
\label{subsec:exp3}

The two $7$-DoF arms of a Unitree H1-2 humanoid hold a loaded tray under
the closed-chain grasp and transport it to a target pose with a $0.40$~$m$
translation and a $0.698$~$rad$ yaw rotation, while a hand-held obstacle moves
into the tray path. Here, $\boldsymbol{h}(\boldsymbol{q})$ includes the
tray--obstacle margin with $h_{\mathrm{safe}}=30$~$mm$ and arm--torso
self-collision margins with $h_{\mathrm{safe}}=25$~$mm$.

As shown in Fig.~\ref{fig:exp3_h12_dynamic_obstacle}, all ten trials succeeded without obstacle contact or
tray loss, with the obstacle approaching at peak speeds of up to approximately
$0.35$~$m/s$. The final translation and orientation errors were
$7.311\pm0.269$~$mm$ and $1.364\pm0.309$~$mrad$, respectively, and the
arm--torso clearance remained above the prescribed $25$~$mm$ margin, with a
per-trial minimum of $25.555\pm0.446$~$mm$ and an overall minimum of
$25.109$~$mm$.

Because $\boldsymbol{h}(\boldsymbol{q})$ is time-invariant, each MPPI
update treats the obstacle as stationary, so its motion between updates can
drive the measured state into the safety buffer, after which the CBF restoring
term steers it back. The peak buffer penetration was $14.287\pm4.666$~$mm$,
reaching at most $19.239$~$mm$ of the $30$~$mm$ buffer at $0.35$~$m/s$. An
obstacle approaching faster than the CBF restoring rate, or advancing beyond
the remaining buffer within one update interval, could therefore reach the
tray, which bounds the operating range of the time-invariant formulation and
is common to reactive controllers without obstacle prediction. When the
obstacle motion can be predicted, it could be incorporated into the inequality
constraints as a time-varying margin, allowing the arms to retreat before the
obstacle enters the buffer.

\section{Conclusion}
\label{sec:conclusion}
We presented PR-MPPI, which enforces equality and inequality constraints
inside the rollouts through a nested velocity-space projection and retracts
the returned command onto the equality to numerical tolerance, independently
of the task weighting. On a $14$-DoF dual-arm system, the returned commands
reached this tolerance while the measured trajectories retained the
inequality margins under aggressive task costs, and the same pipeline was
deployed on a Unitree H1-2 humanoid for tray transport with reactive obstacle avoidance.

The proposed formulation operates at the kinematic level: the projection and retraction act on configurations and velocities, and a low-level controller tracks the retracted command. The structure itself, however, is not tied to this choice. Sampled torques can equally be projected onto the dynamically consistent tangent space of the same constraint manifold, with torque limits and contact-force margins entering as inequalities of the same nested form. In addition, the inequalities are enforced to first order through the CBF condition, which admits a lightweight restricted solve at every rollout step. Augmented-Lagrangian or primal-dual approaches could treat them without this linearization, at the cost of iterative solves. These extensions are left to future work. A remaining limitation is that the inequalities are treated as time-invariant within each update, so environments that change rapidly and unpredictably can consume the safety buffer between replanning updates.

\bibliographystyle{IEEEtran}
\bibliography{IEEEabrv, references}

\end{document}